\pdfoutput=1

\documentclass[11pt]{article}

\usepackage{EMNLP2023}

\usepackage{times}
\usepackage{latexsym}

\usepackage[T1]{fontenc}

\DeclareCaptionLabelFormat{andtable}{#1~#2  \&  \tablename~\thetable}
\usepackage{amsmath,amsfonts,bm}
\usepackage{algorithm}
\usepackage[noend]{algpseudocode}
\usepackage{xspace}

\def\eqref#1{equation~\ref{#1}}

\def\1{\bm{1}}

\def\rvk{{\mathbf{k}}}

\def\rvq{{\mathbf{q}}}

\def\rvs{{\mathbf{s}}}

\def\rvv{{\mathbf{v}}}

\def\rvx{{\mathbf{x}}}
\def\rvy{{\mathbf{y}}}

\def\mB{{\bm{B}}}
\def\mC{{\bm{C}}}

\DeclareMathAlphabet{\mathsfit}{\encodingdefault}{\sfdefault}{m}{sl}
\SetMathAlphabet{\mathsfit}{bold}{\encodingdefault}{\sfdefault}{bx}{n}

\usepackage[utf8]{inputenc}

\usepackage{microtype}

\usepackage{inconsolata}

\usepackage{hyperref}       
\usepackage{url}            
\usepackage{booktabs}       
\usepackage{amsfonts}       
\usepackage{nicefrac}       
\usepackage{graphicx}
\usepackage{amsmath}
\usepackage{amsthm}
\usepackage{amssymb}
\usepackage{multirow}
\usepackage[english]{babel}
\usepackage{bm}

\NewDocumentCommand{\rot}{O{45} O{1em} m}{\makebox[#2][l]{\rotatebox{#1}{#3}}}%

\title{
\setlength{\tabcolsep}{3pt}
MossNet: Mixture of State-Space Experts is a \\ Multi-Head Attention
}

\author{%
  Shikhar Tuli, James Smith, Haris Jeelani, Chi-Heng Lin, Abhishek Patel, \\ {\bf Vasili Ramanishka, Yen-Chang Hsu, Hongxia Jin} \\
  Samsung Research America \\
  665 Clyde Ave, Mountain View, CA 94043 \\
}

\begin{document}
\maketitle

\begin{abstract}
  Large language models (LLMs) have significantly advanced generative applications in natural language processing (NLP). Recent trends in model architectures revolve around efficient variants of transformers or state-space/gated-recurrent models (SSMs, GRMs). However, prevailing SSM/GRM-based methods often emulate only a single attention head, potentially limiting their expressiveness. In this work, we propose \textbf{MossNet}, a novel \underline{m}ixture-\underline{o}f-\underline{s}tate-\underline{s}pace-experts architecture that emulates a linear multi-head attention (MHA). MossNet leverages a mixture-of-experts (MoE) implementation not only in channel-mixing multi-layered perceptron (MLP) blocks but also in the time-mixing SSM kernels to realize multiple ``attention heads.'' Extensive experiments on language modeling and downstream evaluations show that MossNet outperforms both transformer- and SSM-based architectures of similar model size and data budgets. Larger variants of MossNet, trained on trillions of tokens, further confirm its scalability and superior performance. In addition, real-device profiling on a Samsung Galaxy S24 Ultra and an Nvidia A100 GPU demonstrate favorable runtime speed and resource usage compared to similarly sized baselines. Our results suggest that MossNet is a compelling new direction for efficient, high-performing recurrent LLM architectures.
\end{abstract}

\section{Introduction}

Rapid advancements in training and deployment of foundation models have revolutionized various generative applications, including the development of sophisticated chatbots~\citep{chatgpt}, generation of video~\citep{sora}, coding assistance~\citep{code-llama}, and robotic manipulation~\citep{rt2}. With an increasing number of LLM architectures being proposed, such as transformers~\citep{vaswani,gpt}, SSMs~\citep{lssl,s4}, and linear GRMs~\citep{gateloop,hgrn}\footnote{Although SSMs can be considered a specific subset of GRMs, we distinguish them due to their distinct terminology in the literature and their basis in state-space theory, encompassing both continuous-time systems and their discretization.}, the field of NLP continues to evolve at a remarkable pace. The continuous development of these models presents both opportunities and challenges.

\subsection{Challenges and Motivation}

Transformers, introduced by~\citet{vaswani}, have been particularly influential in NLP due to their success in language modeling. The transformer architecture relies on a stack of MHA and MLP blocks. Despite their effectiveness, transformers face several efficiency challenges, including an inability to model outside the context window (although, recent works attempt to mitigate this;~\citealt{infini-attention}), quadratic scaling of compute, and linear scaling of cache with respect to context length. Efficient variants have been proposed that attempt to overcome these drawbacks~\citep{efficient-trnsformers-survey}. Other recent works aim to improve efficiency by replacing the MLP block with a mixture-of-expert (MLP-MoE) block~\citep{switch-transformer,mixtral} or the MHA block with a mixture-of-attention (MHA-MoA) block~\citep{moa}. However, these solutions often trade {\color{black} trade performance for efficiency}.

SSMs, along with recently proposed GRMs, present a promising alternative to the transformer architecture, offering better computational and memory efficiency due to their inherent recurrent design.~\citet{mamba} introduced Mamba, a hardware-optimized selective SSM that achieves high efficiency without sacrificing performance, thanks to the work-efficient parallel scan algorithm~\citep{blelloch,parallelizing-linear-rnn}. Recently proposed extensions of the Mamba architecture, including BlackMamba/MoE-Mamba~\citep{blackmamba,moe-mamba} and Jamba~\citep{jamba}, along with other parallely-proposed GRMs~\citep{rwkv,retnet,gateloop,griffin} match the performance of transformers while maintaining the benefits of recurrent models. More importantly, these works show that such models can emulate the self-attention operation in their mathematical parallel formulation, albeit only a single attention head.

\subsection{Our Contribution}

Due to the single attention head modeling in existing SSMs and GRMs, they exhibit many performance drawbacks~\cite{repeat-after-me,jamba,mamba-360}. Hence, in this work, we propose MossNet, a robust and scalable alternative to current LLM architectures based on a mixture of state-space experts, addressing key challenges, and pushing the boundaries of what is achievable with SSMs. MossNet attempts to model an MHA by extending an existing SSM architecture. More concretely, we summarize the contributions of this work next.

\begin{itemize}
    \item We propose MossNet, a novel architecture that models not just a single self-attention head but an MHA (specifically, its linear mixture-of-expert implementation, i.e., MHA-MoA), just like state-of-the-art transformer models (with linear attention). We mathematically show how a mixture of state-space experts models an MHA. 
    \item We do a \emph{fair} comparison of recently-proposed LLM architectures based on perplexity (PPL) and downstream benchmark performance for small-scale models. Through rigorous experimentation, we empirically show how MossNet outperforms other popular transformer- and SSM/GRM-based baselines.
    \item We train larger variants of MossNet models, namely MossNet-8x200M+, and compare it against state-of-the-art baselines of similar active and total parameter counts. MossNet-8x200M+, in top-2 mode, outperforms Qwen2.5-0.5B by a significant margin, despite being trained on a fraction of pre-training tokens.
    \item We profile the prefill and generation speed of the proposed MossNet models on a Samsung Galaxy S24 Ultra smartphone and an Nvidia A100 GPU. {\color{black} On resource-constrained devices, MossNet-8x200M+ is \emph{significantly} faster in terms of prefill and generation speed along with memory consumption when compared to transformers- and SSM-based baselines with similar active parameter counts.}
\end{itemize}

The rest of the article is organized as follows. Section~\ref{sec:method} details the MossNet architecture along with the proposed evaluation methods. Section~\ref{sec:experiments} presents the experimental results. 
Finally, Section~\ref{sec:conclusion} concludes the article and Section~\ref{sec:limitations} provides the limitations.

\section{Method}
\label{sec:method}

In this section, we discuss the implementation details of the MossNet model.

\subsection{Preliminaries}

We now discuss the required background on the Mamba architecture and the traditional MoE implementation in models like Mixtral-8x7B~\citep{mixtral} and BlackMamba/MoE-Mamba~\citep{blackmamba,moe-mamba}.

\subsubsection{Mamba}

SSMs are a class of sequence models with linear complexity with respect to the sequence length. This results in superior efficiency, especially for long-context input. Multi-dimensional SSMs are defined using four parameters $\bm{\Delta}$, $\bm{A}$, $\bm{B}$, and $\bm{C}$, and sequence-to-sequence transformations from $\rvx_t \in \mathbb{R}^N$ to $\rvy_t \in \mathbb{R}^M$ through an implicit latent state $\rvs_t \in \mathbb{R}^P$ as follows~\cite{s4},
\begin{align}
    \rvs'_t &= \bm{A} \rvs_t + \bm{B} \rvx_t \label{eq:ssm_cont_eq_1} \\
    \rvy_t &= \bm{C} \rvs_t \label{eq:ssm_cont_eq_2}
\end{align}
where, $\bm{A} \in \mathbb{R}^{P \times P}$, $\bm{B} \in \mathbb{R}^{P \times N}$, $\bm{C} \in \mathbb{R}^{M \times P}$ {\color{black}, and $\rvs'_t$ is the derivative of $\rvs_t$}. In its discrete parameterization,
\begin{align}
    \rvs_t &= \bm{\bar{A}} \rvs_{t-1} + \bm{\bar{B}} \rvx_t \label{eq:ssm_disc_eq_1} \\
    \rvy_t &= \bm{C} \rvs_t \label{eq:ssm_disc_eq_2}
\end{align}
where,
\begin{align}
    \bm{\bar{A}} &= \exp{(\bm{\Delta} \bm{A})}, \text{and} \\
    \bm{\bar{B}} &= (\bm{\Delta} \bm{A})^{-1} (\exp{(\bm{\Delta} \bm{A})} - \bm{I}) \cdot \bm{\Delta} \bm{B}.
\end{align}

One can efficiently compute a linear dynamical system like this in parallel via a convolution~\cite{s4} or parallel associative scan~\cite{blelloch}. On the other hand, one can leverage the recurrent form presented above for rapid generation at inference time. Mamba~\cite{mamba} makes the discrete parameters input-dependent, i.e., $\bm{\bar{A}}_t$, $\bm{\bar{B}}_t$, and $\bm{C}_t$.

\citet{mamba} offer an intuitive interpretation of these parameters. $\bm{\bar{A}}$ controls the transition dynamics, while $\bm{\bar{B}}$ and $\bm{C}$ control the selectivity of the input $x_t$ into the hidden state $h_t$ and the state into the output $y_t$, respectively. Finally, $\bm{\Delta}$ controls the balance between how much to focus or ignore the current input $x_t$. However, in an MHA, each head focuses on different aspects of the relationships between words/tokens. An MHA thus provides enhanced expressiveness, mitigates information loss, and improves learning capability compared to a single attention head. In the \emph{same} spirit, we hypothesize that in an SSM, there should be multiple such parameters that focus on different parts of the input sequence. For instance, multiple $\bm{\Delta}$'s could focus on the selectivity of the current input in the context of multiple dependencies in data.

\subsubsection{Mixture of Experts}

Primarily, MoEs are synonymous with MLP layers within a transformer model (we call this the MLP-MoE block;~\citealt{switch-transformer}). Such models reduce the inference cost by routing tokens to specific MLP \emph{experts}. A router maps the token representations to experts, where each expert is simply a standard transformer MLP block. The expert to whom the token is routed is chosen from the top-$k$ of the expert probabilities, where $k$ is a hyperparameter. Mathematically, an input $\rvx_t$ is mapped through the router to a probability distribution $p_i(\rvx_t)$, where $i$ labels the experts. Upon selecting the top-$k$ probabilities, the output of the MoE layer at time-step $t$, i.e., $\rvy_t$ is a linearly weighted combination of each expert's computation on the input,
\begin{equation}
    \rvy_t = \sum_{i \in \text{top-}k} p_i(\rvx_t) \bm{E}_i(\rvx_t) \label{eq:moe_sum}
\end{equation}
where $\bm{E}_i$ is the $i$-th MLP expert.

Instead of applying MoE to the, channel-mixing, MLP layers, \citet{moa} apply the MoE to the, time-mixing, MHA blocks (we call this the MHA-MoA block). This block performs as well as the traditional MHA, while providing the benefits of MoE~\cite{switch-transformer}. We take inspiration from the MHA-MoA block in order to emulate multiple \emph{attention} heads in the proposed MossNet architecture.

\subsection{MossNet Architecture}

\begin{figure}[t]
    \centering
    \includegraphics[width=\linewidth]{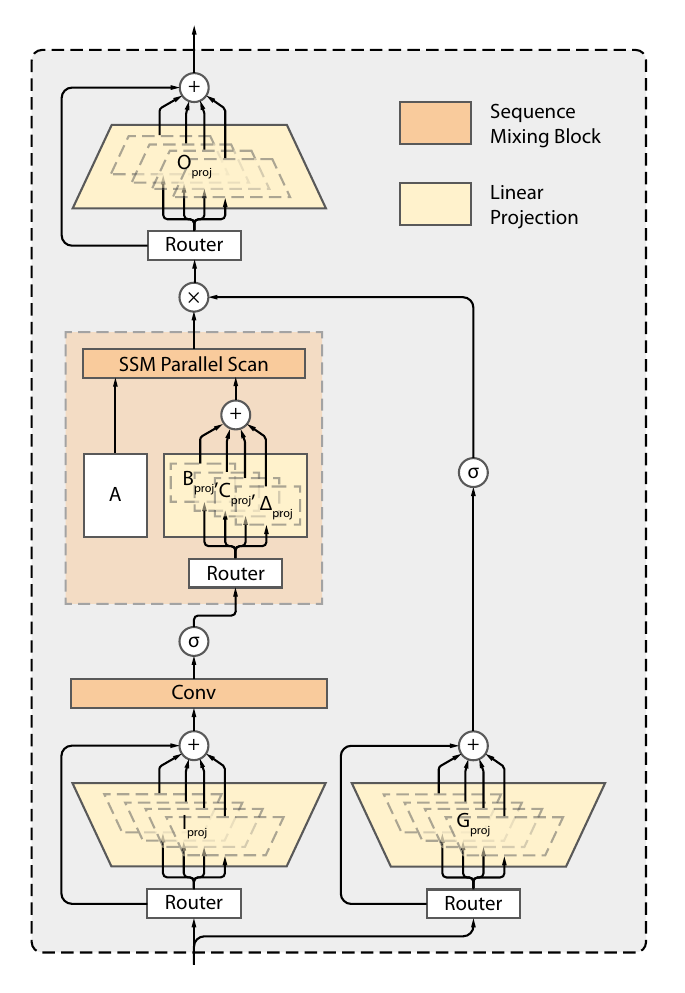}
    \caption{Simplified working schematic of the MossNet block. We implement MoE in channel mixing input, gate, and output projections and time mixing input-dependent SSM parameters $\bm{B}$, $\bm{C}$, and $\bm{\Delta}$.}
    \label{fig:MossNet_arch}
\end{figure}

MossNet extends the Mamba architecture~\cite{mamba} by implementing MoE in various projection operations. Fig.~\ref{fig:MossNet_arch} shows a working schematic of the MossNet block. Concretely, we implement MoE for the channel-mixing linear projections ($\bm{I}$, $\bm{G}$, and $\bm{O}$) and the sequence transformation input-dependent SSM parameters $\bm{B}$, $\bm{C}$, and $\bm{\Delta}$. The input-independent parameter $\bm{A}$, along with $\bm{B}$ and $\bm{\Delta}$, are used to calculate the discrete SSM parameters $\bm{\bar{A}}$ and $\bm{\bar{B}}$. The combined contribution of the mixture of state-space experts is input to the hardware-optimized SSM parallel scan kernel~\cite{mamba}.

We follow~\citet{switch-transformer} to implement the router network for the MoE implementation. Concretely, the router (implemented as a feed-forward layer) calculates the score $h(\rvx_t) \in \mathbb{R}^{N_\text{experts}}$, where $N_\text{experts}$ is the number of experts. We normalize the scores using a softmax operation to obtain $p_i(\rvx_t)$ in Eq.~(\ref{eq:moe_sum}). For equiproportionate distribution of tokens to the experts, we employ the load balancing loss~\cite{switch-transformer} and add it to the training objective with a weight $\alpha$.

\newtheorem{theorem}{Theorem}
\begin{theorem}
A mixture-of-expert implementation of $\bm{\bar{A}}$, $\bm{\bar{B}}$, and $\bm{C}$ is equivalent to a mixture-of-expert implementation of a linear multi-head attention.
\end{theorem}

\begin{proof}[Proof]
    Recall that the state evolution of a discretely parameterized multi-dimensional selective SSM is 
    \begin{align}
        \rvs_t &= \bm{\bar{A}}_t \rvs_{t-1} + \bm{\bar{B}}_t \rvx_t \\
        \rvy_t &= \bm{C}_t \rvs_t \label{eq:evol}.
    \end{align}
    Expanding Eq.~(\ref{eq:evol}), we get
    \begin{align}
        \rvy_t &=  \sum_{i=1}^t \bm{C}_t\prod_{j=i+1}^t \left(\bm{\bar{A}}_j\right)\bm{\bar{B}}_i \rvx_i  \nonumber \\
        =~& \sum_{i=1}^t \left( \bm{C}_t \prod_{j=1}^t \bm{\bar{A}}_j \right) \left( \prod_{j=1}^i \bm{\bar{A}}^{-1}_j\bm{B}_i  \right) \rvx_i \label{eq:qkv}.
    \end{align}
    On the other hand, the mixture-of-expert implementations of $\bm{\bar{B}}_t$, and $\bm{C}_t$ can be written as, 
    \begin{align}
        \bar{\bm{{B}}_t} &= \sum_{m=1}^{N_\text{experts}} p_m(\rvx_t) \bar{\mB}^m_t, \label{eq:b} \\
        {\bm{{C}}_t} &= \sum_{n=1}^{N_\text{experts}} p_n(\rvx_t) {\mC}^n_t, \label{eq:c}
    \end{align}
    where the experts are functions of input $\rvx_t$. Now, plugging in Eqs.~(\ref{eq:b}) and~(\ref{eq:c}) into Eq.~(\ref{eq:qkv}), we obtain the output at time $t$,
    \begin{align}
        \rvy_t=\sum_{m,n=1}^{N_\text{experts}}\sum_{i=1}^t \left( p_m(\rvx_t) \bm{C}^m_t \prod_{j=1}^t \bm{\bar{A}}_j \right) \times \nonumber\\
        \left( p_n(\rvx_t) \prod_{j=1}^i \bm{\bar{A}}^{-1}_j \bar{\mB}^n_i \right) \rvx_i.
    \end{align}
    If we define $$\rvq^m_t = p_m(\rvx_t) \bm{C}^m_t \left(\prod_{j=1}^t \bm{\bar{A}}_j \right),$$
    $$\rvk^n_i = p_n(\rvx_t) \left(\prod_{j=1}^i  \bm{\bar{A}}^{-1}_j \right) \bar{\mB}^n_i, \rvv_i=\rvx_i,$$ then we can put the expression of the output into a form of a weighted, linear MHA-MoA:
    \begin{align*}
        \rvy_t = \sum_{m,n = 1}^{N_\text{experts}} \sum_{i = 1}^t \langle \rvq^m_t, \rvk^n_i \rangle \rvv_i = \sum_{m,n = 1}^{N_\text{experts}} \text{Attention}_{m,n},
    \end{align*}
    where we interpret $\rvq^m$, $\rvk$, and $\rvv$ as the $m$-th head's query vector, the $n$-th head's key vector, and the shared value vector for all heads, respectively. 
    Finally, we remark that the above expression does not use an output projection since the value vector is shared and equal to $\rvx_t$ for all heads. 
\end{proof}

\paragraph{Remark 1} The above expression differs from the traditional MHA in three aspects: 1) Each head's query interacts with the keys from all other heads, contrasting with the standard approach where queries interact only with their corresponding keys. 2) The key and query are functions of the router probabilities, making them non-linear functions of the input.
3) The value vector is shared among all heads, which eliminates the need for an output projection. 

\paragraph{Remark 2} The above expression differs from the MHA-MoA implementation by \citet{moa} in that the router leverages multiple query and key experts, instead of multiple query experts and common key/value experts.

\paragraph{Remark 3} In the above formulation, we neglect the MoE implementation of $\bm{\bar{A}}$, a function of $\bm{A}$ and $\bm{\Delta}$, for simplicity. In MossNet, we implement $\bm{\Delta}$ as an MoE as well. This would be equivalent to the above formulation, however, at the cost of a more complex set of equations.

\begin{figure*}[t!]
    \centering
    \includegraphics[width=\linewidth]{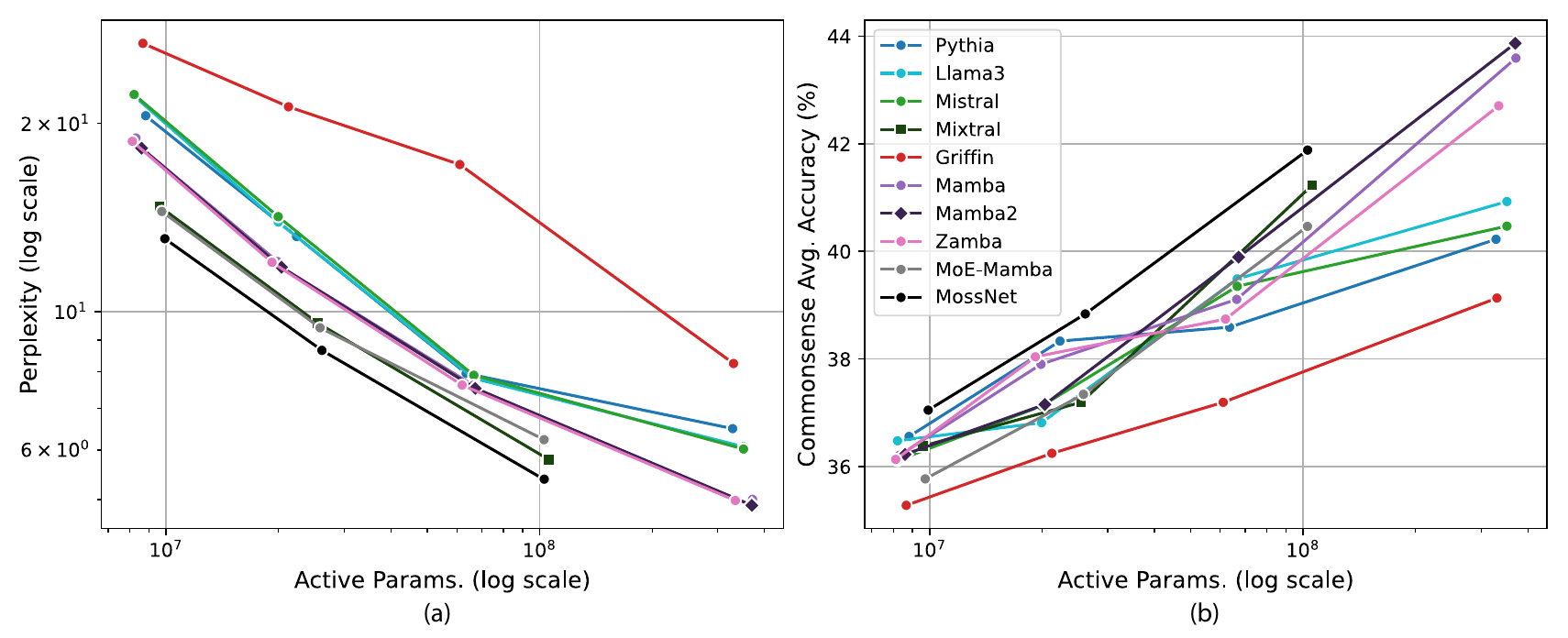}
    \caption{(a) Perplexity and (b) commonsense average accuracy scaling for \emph{fairly-trained} models.}
    \label{fig:commonsense_ppl_scaling}
\end{figure*}

\begin{table*}[t]
\caption{Performance of MossNet and other baselines on SWDE (zero‑shot accuracy), FDA (zero‑shot accuracy), TriviaQA (closed‑book, zero‑shot accuracy), SQuADv2 (zero‑shot F1), RACE (zero‑shot accuracy) and MMLU (five‑shot accuracy) benchmarks.}
\resizebox{\linewidth}{!}{
\begin{tabular}{@{}l|cc|c|cc|cccc|c@{}}
\toprule
\multirow{2}{*}{\textbf{Model}} &
\multicolumn{2}{c|}{\textbf{Recall}} &
\textbf{Closed‑book} &
\multicolumn{2}{c|}{\textbf{Reading Comprehension}} &
\multicolumn{4}{c|}{\textbf{MMLU}} &
\multirow{2}{*}{\textbf{Avg.}}\\ \cmidrule(l){2-10}
& \textbf{SWDE} & \textbf{FDA} & \textbf{TriviaQA} &
\textbf{SQuADv2} & \textbf{RACE} &
\textbf{Hum.} & \textbf{Social Sci.} & \textbf{STEM} & \textbf{Other} &
\\ \midrule
Pythia-9M                       & 0.9 & 0.0 & 0.0 & 11.1 & 23.1 & 24.7 & 22.8 & 23.2 & 23.4 & 14.4 \\
Llama-8M                        & 1.4 & \textbf{0.2} & 0.0 & 4.2 & 23.6 & 24.9 & 23.8 & \textbf{26.9} & 25.6 & 14.5 \\
Mistral-8M                      & 0.8 & 0.1 & 0.0 & 11.8 & 24.3 & 24.5 & 24.0 & 22.9 & 24.2 & 14.7 \\
Mixtral-8x8M                    & 0.1 & 0.0 & 0.1 & 15.5 & 23.4 & 23.9 & 21.7 & 22.4 & 25.4 & 14.7 \\
Griffin-9M                      & 1.1 & 0.0 & 0.0 & 20.6 & 22.5 & 24.2 & 21.7 & 21.3 & 24.0 & 15.0 \\
Mamba-8M                        & 0.6 & 0.0 & 0.0 & 9.8 & 24.0 & 24.5 & 22.0 & 22.7 & 23.6 & 14.1 \\
Mamba2-9M                       & 1.3 & 0.1 & \textbf{0.1} & 6.8 & \textbf{25.4} & 24.1 & 22.4 & 22.3 & 24.8 & 14.1 \\
Zamba-8M                        & \textbf{2.0} & 0.1 & 0.1 & 31.1 & 23.2 & 23.9 & 24.6 & 26.9 & 25.2 & 17.5 \\
MoE‑Mamba‑8x8M                  & 1.1 & 0.0 & 0.0 & \textbf{37.8} & 22.8 & 24.4 & 23.1 & 23.0 & 24.5 & 17.4 \\
\rowcolor{gray!20}
MossNet‑8x8M                    & 1.4 & \textbf{0.2} & 0.0 & 34.7 & 24.4 & \textbf{25.1} & \textbf{25.0} & 24.9 & \textbf{25.8} & \textbf{17.9} \\ \midrule
Pythia‑22M                      & 0.8 & 0.3 & 0.0 & 21.6 & 25.4 & 25.4 & 23.4 & 27.4 & 24.1 & 16.5 \\
Llama‑20M                       & 2.3 & 0.1 & 0.1 & 15.2 & 24.4 & 24.1 & 22.4 & 23.9 & 23.3 & 15.1 \\
Mistral‑20M                     & 0.7 & 0.0 & 0.1 & 5.5 & 24.2 & 23.9 & 21.9 & 23.0 & 23.8 & 13.7 \\
Mixtral‑8x20M                   & 1.1 & 0.0 & 0.3 & 5.5 & 23.6 & 24.8 & 23.8 & 25.0 & 24.4 & 14.3 \\
Griffin‑22M                     & 0.9 & 0.1 & 0.0 & 25.5 & 21.2 & 24.2 & 21.9 & 22.5 & 23.6 & 15.5 \\
Mamba‑20M                       & 0.9 & 0.1 & 0.2 & 6.1 & 22.7 & 24.2 & 22.5 & 22.0 & 23.8 & 13.6 \\
Mamba2‑20M                      & 0.9 & 0.2 & 0.1 & 4.8 & 25.6 & 24.6 & 23.8 & 27.6 & 24.0 & 14.8 \\
Zamba‑20M                       & 4.7 & \textbf{0.5} & 0.1 & 9.2 & 24.3 & 24.0 & 23.3 & 25.8 & 27.8 & 15.5 \\
MoE‑Mamba‑8x20M                 & 3.1 & 0.0 & 0.4 & 1.9 & 25.3 & \textbf{26.3} & 24.3 & 25.3 & 24.4 & 14.7 \\
\rowcolor{gray!20}
MossNet‑8x20M                   & \textbf{4.8} & 0.3 & \textbf{0.4} & \textbf{26.3} & \textbf{25.8} & 24.3 & \textbf{26.0} & \textbf{28.8} & \textbf{27.9} & \textbf{20.1} \\ \midrule
Pythia‑64M                      & 5.6 & 0.5 & 0.8 & 21.3 & 27.6 & 24.8 & 24.6 & \textbf{28.4} & 24.5 & 16.5 \\
Llama‑67M                       & 7.8 & 0.4 & 0.8 & 13.3 & 25.8 & 24.1 & 24.6 & 26.2 & 25.8 & 16.5 \\
Mistral‑67M                     & 0.3 & 0.0 & 0.6 & 5.0 & 24.0 & 24.1 & 23.2 & 26.4 & 22.7 & 14.0 \\
Mixtral‑8x67M                   & 0.5 & 0.0 & 2.2 & 16.1 & 26.0 & 24.4 & 24.5 & 22.3 & 22.5 & 15.4 \\
Griffin‑61M                     & 0.5 & 0.0 & 0.0 & 32.3 & 23.6 & 24.3 & 22.0 & 21.5 & 23.8 & 16.4 \\
Mamba‑66M                       & 3.7 & 0.3 & 0.7 & 3.7 & 25.6 & 25.1 & 23.5 & 25.0 & 23.2 & 14.5 \\
Mamba2‑67M                      & 3.3 & 0.2 & 0.5 & 2.5 & 25.6 & \textbf{25.4} & 24.4 & 25.9 & 25.5 & 14.8 \\
Zamba‑62M                       & 8.3 & 0.4 & 1.2 & 3.3 & 25.7 & 24.7 & 23.8 & 26.9 & 25.5 & 15.5 \\
MoE‑Mamba‑8x66M                 & 5.0 & 0.6 & 1.1 & 3.0 & 25.9 & 25.1 & 23.4 & 25.0 & 23.1 & 14.7 \\
\rowcolor{gray!20}
MossNet‑8x66M                   & \textbf{13.0} & \textbf{1.4} & \textbf{2.9} & \textbf{34.4} & \textbf{27.9} & 25.2 & \textbf{24.8} & 25.4 & \textbf{25.9} & \textbf{20.1} \\ \midrule
{\color{gray} Pythia‑330M}  & {\color{gray} 11.0} & {\color{gray} 0.5} & {\color{gray} 1.1} & {\color{gray} 3.0}  & {\color{gray} 26.7} & {\color{gray} 25.1} & {\color{gray} 28.6} & {\color{gray} 27.6} & {\color{gray} 24.1} & {\color{gray} 16.4} \\
{\color{gray} Llama‑350M}   & {\color{gray} 9.3}  & {\color{gray} 0.7} & {\color{gray} 1.4} & {\color{gray} 4.3}  & {\color{gray} 26.8} & {\color{gray} 24.8} & {\color{gray} \textbf{30.5}} & {\color{gray} 28.1} & {\color{gray} 24.4} & {\color{gray} 16.7} \\
{\color{gray} Mistral‑350M} & {\color{gray} 7.9}  & {\color{gray} 0.0} & {\color{gray} 2.1} & {\color{gray} 3.2}  & {\color{gray} 27.6} & {\color{gray} \textbf{26.6}} & {\color{gray} 25.0} & {\color{gray} 26.8} & {\color{gray} 23.9} & {\color{gray} 15.9} \\
{\color{gray} Griffin‑330M} & {\color{gray} 3.5}  & {\color{gray} 0.1} & {\color{gray} 0.2} & {\color{gray} \textbf{13.8}} & {\color{gray} 23.4} & {\color{gray} 25.4} & {\color{gray} 24.6} & {\color{gray} 27.1} & {\color{gray} 24.9} & {\color{gray} 15.9} \\
{\color{gray} Mamba‑370M}   & {\color{gray} 4.8}  & {\color{gray} 0.4} & {\color{gray} 1.7} & {\color{gray} 4.4}  & {\color{gray} 27.1} & {\color{gray} 25.8} & {\color{gray} 25.2} & {\color{gray} \textbf{28.2}} & {\color{gray} 24.1} & {\color{gray} 15.7} \\
{\color{gray} Mamba2‑370M}  & {\color{gray} 8.2}  & {\color{gray} 0.8} & {\color{gray} 1.8} & {\color{gray} 3.4}  & {\color{gray} \textbf{28.9}} & {\color{gray} 24.4} & {\color{gray} 22.6} & {\color{gray} 23.1} & {\color{gray} 25.4} & {\color{gray} 15.4} \\
{\color{gray} Zamba‑330M}   & {\color{gray} \textbf{19.7}} & {\color{gray} \textbf{6.4}} & {\color{gray} \textbf{2.4}} & {\color{gray} 9.1} & {\color{gray} 27.9} & {\color{gray} 24.1} & {\color{gray} 29.5} & {\color{gray} 26.5} & {\color{gray} \textbf{25.7}} & {\color{gray} \textbf{19.0}} \\ \bottomrule
\end{tabular}}
\label{tbl:fair_other_eval}
\end{table*}

\subsection{Training and Evaluation Setup}

To \emph{fairly} compare different architectures, we train a suite of models with varying number of parameters on the same language modeling dataset. Concretely, we compare the performance of various architectures. These include three popular transformer architectures: Pythia~\cite{pythia}, Llama~\cite{llama2}, Mistral~\cite{mistral} and its MoE extension Mixtral~\cite{mixtral}, a recently-proposed GRM, i.e., Griffin~\cite{griffin}, along with Mamba~\cite{mamba} and its extensions: Zamba~\cite{zamba} and MoE-Mamba~\cite{blackmamba,moe-mamba}. We also add recently-proposed Mamba2~\citep{mamba2} to our comparisons. We use the same BPE tokenizer for all models~\cite{gpt_neo_x}. We train these models on the Cosmopedia~\cite{cosmopedia} dataset, which has shown high model performance per pre-training token. We present additional model hyperparameters along with other training details in Appendix~\ref{app:hyperparameters}.

Following the \emph{fairly-trained} setting, we train a larger model, namely MossNet-8x200M+, on a custom dataset with 2.8T tokens comprised of a mixture of existing open-source datasets. We describe the hyperparameter choices and the training recipes employed for this model in Appendix~\ref{app:hyperparameters}. We provide details of the custom pre-training dataset in Appendix~\ref{app:custom_dataset}.

We compare the performance of the proposed MossNetsuite of model against baselines on various downstream benchmarks. 

\section{Experiments}
\label{sec:experiments}

In this section, we present experimental results comparing the proposed MossNet suite of models against \emph{fairly-trained} and state-of-the-art baselines.

\subsection{Downstream Language Modeling Performance}

First, we evaluate the MossNet architecture, along with other baselines, based on language modeling perplexity on the Cosmopedia dataset and consider eight standard commonsense reasoning benchmarks: ARC challenge (ARC-c) and ARC easy 
(ARC-e,~\citealt{arc}), 
BoolQ~\citep{boolq}, COPA~\citep{copa}, HellaSwag~\citep{hellaswag}, OpenBookQA 
(OBQA,~\citealt{obqa}), PIQA~\citep{piqa}, and WinoGrande~\citep{winogrande}. 
We perform evaluations in the zero-shot setting as done in the language modeling community. 
We \emph{fairly} train all models on the same dataset and under the same setting (more details in Appendix~\ref{app:hyperparameters}).

Fig.~\ref{fig:commonsense_ppl_scaling} shows how the performance scales for MossNet and other baselines, both dense and sparse. MossNet achieves lower perplexity and higher average commonsense accuracy, showing superior scaling across model sizes. This shows the advantages of multiple state-space ``heads'' in language modeling performance.

We also evaluate MossNet and other baselines on more benchmarks: infromation retrieval on SWDE and FDA~\citep{based}, closed-book question answering on TriviaQA~\citep{triviaqa}, reading comprehension on SQuADv2~\citep{squad} and RACE~\citep{race}, and general knowledge and reasoning on MMLU~\citep{mmlu}. Table~\ref{tbl:fair_other_eval} shows the results. MossNet outperforms baselines with similar number of active parameters on most benchmarks.

\begin{table*}
\centering
\caption{Performance of MossNet and state-of-the-art baselines on ARC-c (zero-shot accuracy), ARC-e (zero-shot accuracy), HellaSwag (zero-shot accuracy), PIQA (zero-shot accuracy), WinoGrande (zero-shot accuracy), SQuADv2 (zero-shot F1 score), and MMLU (five-shot accuracy) benchmarks. We evaluate the instruction-tuned models wherever available. *We evaluate all models except Hymba-350M (not publicly available) using \texttt{lm-evaluation-harness}~\citep{eval-harness}. For Hymba-350M, we present the reported results~\citep{hymba}.}
\resizebox{\linewidth}{!}{%
\begin{tabular}{@{}cl|c|ccccccc|c@{}}
\toprule
 & \textbf{Model} & \textbf{Train Tokens} & \textbf{ARC-c} & \textbf{ARC-e} & \textbf{HellaSwag} & \textbf{PIQA} & \textbf{WinoGrande} & \textbf{SQuADv2} & \textbf{MMLU} & \textbf{Average} \\
\midrule
\multirow{7}{*}{\rot[90]{$\sim$500M}} & Mamba-370M & 0.3T & 28.0 & 55.1 & 46.5 & 69.5 & 55.3 & 5.8 & 23.1 & 40.5 \\
 & Mamba2-370M & 0.3T & 26.9 & 54.9 & 46.9 & 70.5 & 55.7 & 5.9 & 23.6 & 40.6 \\
 & BlackMamba-1.5B & 0.3T & 24.1 & 56.1 & 36.5 & 69.0 & 52.6 & 4.7 & 19.4 & 37.5 \\
 & Hymba-350M* & 1.5T & - & - & 55.1 & 72.9 & 57.8 & - & 34.5 & - \\
 & Qwen2.5-0.5B & 18T & 34.2 & 59.8 & 53.0 & 70.9 & 56.5 & 12.3 & \textbf{47.4} & 47.7 \\
 & SmolLM2-360M & 11T & 34.3 & 49.6 & 57.3 & 72.0 & 57.8 & 9.3 & 26.2 & 43.8 \\
\rowcolor{gray!20} & MossNet-8x200M+ (top-2) & 2.8T & \textbf{41.4} & \textbf{68.4} & \textbf{63.9} & \textbf{76.1} & \textbf{62.5} & \textbf{20.9} & 41.2 & \textbf{53.5} \\
\midrule
\multirow{4}{*}{\rot[90]{$\sim$700M}} & Mamba-790M & 0.3T & 29.5 & 61.2 & 55.1 & 72.1 & 56.1 & 8.5 & 24.0 & 43.8 \\
 & Mamba2-790M & 0.3T & 28.5 & 61.0 & 54.9 & 72.0 & 60.2 & 8.7 & 24.4 & 44.2 \\
 & BlackMamba-2.8B & 0.3T & 24.5 & 60.3 & 39.7 & 71.2 & 52.1 & 6.8 & 22.7 & 39.6 \\
\rowcolor{gray!20} & MossNet-8x200M+ (top-3) & 2.8T & \textbf{40.2} & \textbf{69.8} & \textbf{65.9} & \textbf{76.3} & \textbf{64.2} & \textbf{28.1} & \textbf{43.6} & \textbf{55.4} \\
\midrule
\multirow{6}{*}{\rot[90]{{\color{gray}$\sim$1.5B}}} & {\color{gray} Rene-1.3B} & {\color{gray}1.5T} & {\color{gray}34.4} & {\color{gray}61.7} & {\color{gray}69.4} & {\color{gray}77.5} & {\color{gray}62.9} & {\color{gray}17.4} & {\color{gray}32.6} & {\color{gray}50.8} \\
 & {\color{gray}Hymba-1.5B} & {\color{gray}1.5T} & {\color{gray}44.3} & {\color{gray}76.0} & {\color{gray}71.1} & {\color{gray}77.4} & {\color{gray}66.6} & {\color{gray}20.3} & {\color{gray}52.3} & {\color{gray}58.3} \\
 & {\color{gray}Phi1.5} & {\color{gray}0.15T} & {\color{gray}48.5} & {\color{gray}73.9} & {\color{gray}62.6} & {\color{gray}76.4} & {\color{gray}73.6} & {\color{gray}19.4} & {\color{gray}42.9} & {\color{gray}56.8} \\
 & {\color{gray}DCLM-1.4B} & {\color{gray}4.3T} & {\color{gray}47.5} & {\color{gray}75.7} & {\color{gray}73.2} & {\color{gray}78.5} & {\color{gray}67.5} & {\color{gray}28.9} & {\color{gray}51.6} & {\color{gray}60.4} \\
 & {\color{gray}Qwen2.5-1.5B} & {\color{gray}18T} & {\color{gray}47.0} & {\color{gray}76.7} & {\color{gray}69.0} & {\color{gray}76.8} & {\color{gray}63.5} & {\color{gray}25.2} & {\color{gray}60.9} & {\color{gray}59.9} \\
 & {\color{gray}SmolLM2-1.7B} & {\color{gray}11T} & {\color{gray}44.1} & {\color{gray}63.6} & {\color{gray}72.6} & {\color{gray}77.0} & {\color{gray}69.1} & {\color{gray}19.3} & {\color{gray}49.9} & {\color{gray}56.5} \\
\bottomrule
\end{tabular}%
}
\label{tbl:1b5_eval}
\end{table*}

Finally, we train MossNet-8x200M+ for 2.8T tokens on a custom pretraining dataset and compare it against state-of-the-art baselines. We trained MossNet-8x200M+ to support both top-2 and top-3 modes, resulting in 477M and 657M active parameters, respectively (more details in Section~\ref{app:hyperparameters}). This allows the same model to support low-power and high-power models on-device.
Table~\ref{tbl:1b5_eval} shows the results. In the top-2 mode, we compare MossNet with Mamba-370M~\citep{mamba}, Mamba2-370M~\citep{mamba2}, BlackMamba-1.5B~\citep{blackmamba}, Hymba-350M~\citep{hymba}, Qwen2.5-0.5B~\citep{qwen2.5}, and SmolLM2-360M~\citep{smollm2}. MossNet-8x200M+ outperforms Qwen2.5-0.5B by 5.8\% average accuracy. In the top-3 mode, we compare MossNet with Mamba-790M, Mamba2-790M, and BlackMamba-2.8B. We also see notable improvement going from top-2 to top-3 gating. Thanks to the proposed MoE design and training strategy, MossNet exhibits flexibility in different active-parameter-constrained settings, unlike static models. We also present the performance of state-of-the-art models with active parameters around 1.5B. MossNet not only outperforms baselines with similar active parameters, but also reaches the performance of other models with 1.5B active parameters, {\color{black} while achieving \emph{significant} latency and memory gains as we show next}. 




\subsection{Speed and Memory Profiling}

In this section we present the memory and speed profiling results on server (Nvidia A100-80GB GPU) and on mobile (Samsung Galazy S24 Ultra).

\subsubsection{Server GPU Results}

\begin{figure*}[t!]
    \centering
    \includegraphics[width=\linewidth]{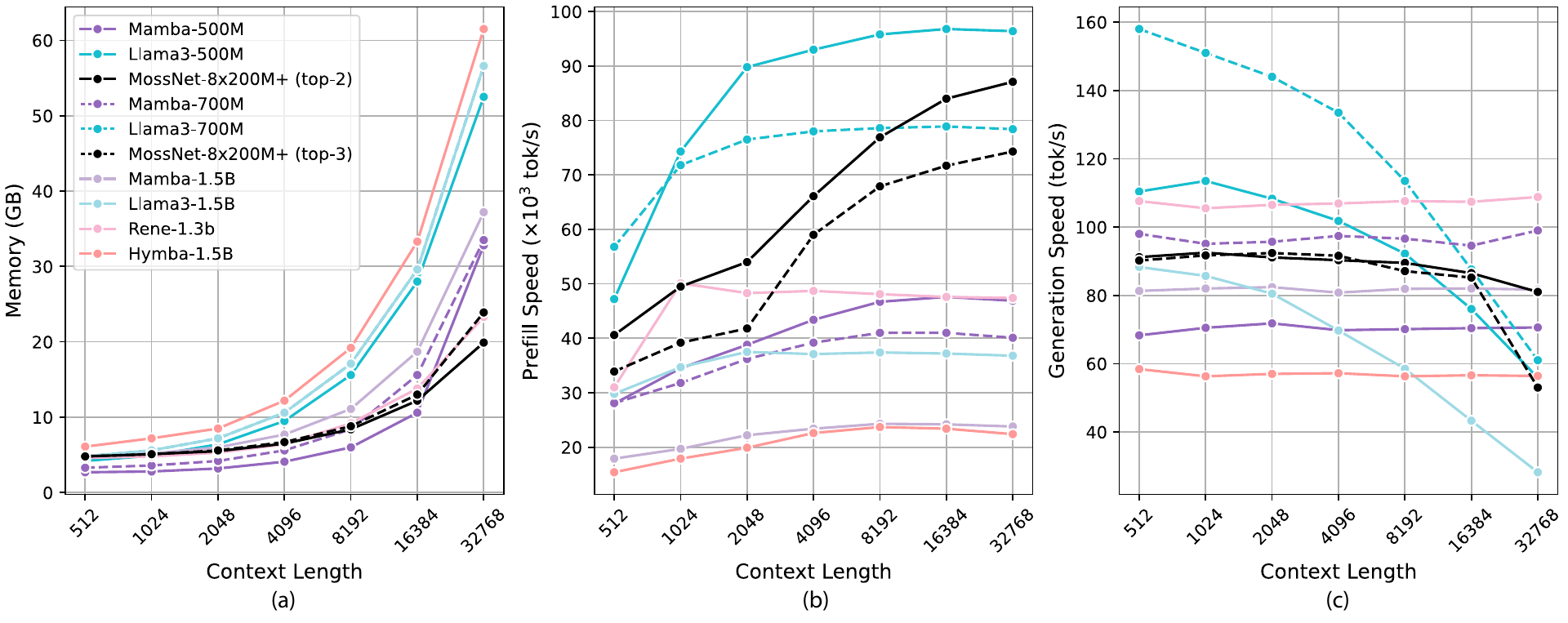}
    \caption{(a) Memory consumption, (b) prefill speed, and (c) generation speed with context length for MossNet-8x200M+ and baselines on A100-80GB (FP16 precision, FlashAttention 2). Batch size set to 4.}
    \label{fig:gpu_speed_memory}
\end{figure*}

Fig.~\ref{fig:gpu_speed_memory} presents (a) memory consumption, (b) prefill speed, and (c) generation speed across increasing context lengths for MossNet-8×200M+ compared to several single‐expert baselines of similar active and total parameter scale. While all models naturally require more GPU memory as context length grows, MossNet’s MoE design contains that growth more effectively, keeping memory usage lower than monolithic baselines with comparable or larger parameter counts. Particularly, for longer contexts, e.g. 32K, MossNet-8x200M+ in top-2 mode achieves the lowest memory usage relative to baselines. MossNet also demonstrates consistently high prefill throughput. Its prefill speed approaches that of Llama3-500M/700M, being far superior to other SMM/hybrid baselines.
Further, as shown in Fig~\ref{fig:gpu_speed_memory}(c), MossNet’s token‐by‐token generation speed remains stable across large contexts, whereas competing baselines often slow down significantly. In short, these GPU‐based results highlight the key advantages of expert routing: more efficient memory usage and stronger large‐context performance, without sacrificing speed.

\begin{figure*}
    \centering
    \includegraphics[width=\linewidth]{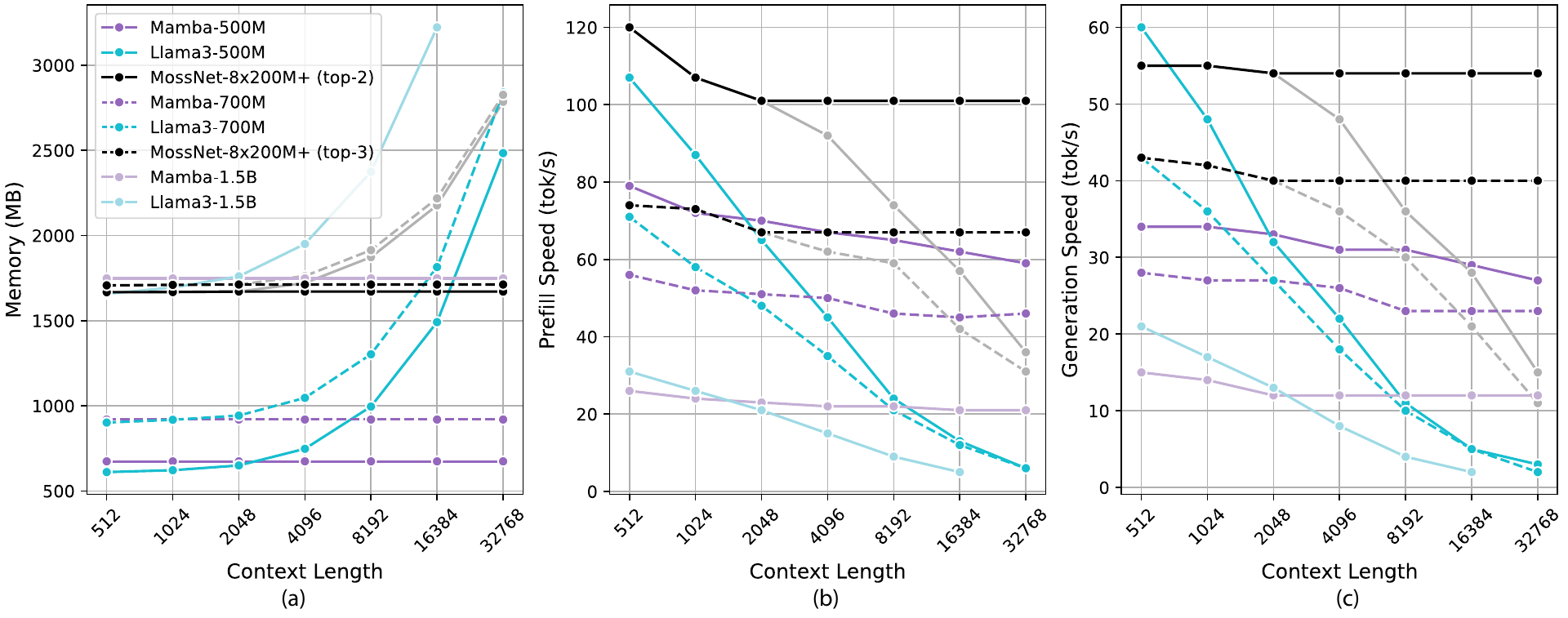}
    \caption{(a) Memory consumption, (b) prefill speed, and (c) generation speed with context length for MossNet-8x200M+ and baselines on Samsung Galaxy S24 Ultra (Q8 precision). Batch size set to 1. Gray line plots depict performance without SWA implemented. Llama3-1.5B results not plotted for 32K context due to out-of-memory error.}
    \label{fig:device_speed_memory}
\end{figure*}

\subsubsection{Mobile Results}

Fig.~\ref{fig:device_speed_memory} illustrates the same three metrics on a Samsung Galaxy S24 Ultra (on CPU with Q8 precision) for a batch size of 1, further underscoring MossNet’s benefits in resource‐constrained edge settings. Here, MossNet’s memory footprint stays essentially flat at around 1.6 GB across all context lengths, while the Llama3 models consume increasingly large amounts of memory as the context grows. Mamba too has a flat memory curve due to serial operation of the scan operation on-device. MossNet’s prefill and generation speeds remain comfortably higher and more consistent than those of baselines, which degrade more severely as context length increases. The drop in prefill speed on mobile device (unlike on server GPU) could be attributed to the lower compute capacity for parallel processing. The stable performance and reduced resource use make MossNet especially suitable for on‐device inference scenarios, where users often demand responsiveness and must operate under strict memory and compute constraints.

\subsection{Architecture Modifications}

\begin{table}[]
\centering
\caption{Effect of architectural modifications to MossNet. Perplexity reported on Cosmopedia evaluation set.}
\resizebox{\linewidth}{!}{
\begin{tabular}{@{}l|ccc@{}}
\toprule
\textbf{Model} & \textbf{Tot. Param. (M)} & \textbf{Act. Param. (M)} & \textbf{PPL} \\ \midrule
MossNet-8x8M & 19.7 & 9.9 & 13.1 \\
\; w/o MHA & 21.4 & 10.3 & 13.5 \\
\; w/o MLP-MoE & 19.1 & 9.8 & 13.4 \\
\; w/ top-1, 8 experts & 19.7 & 8.3 & 15.3 \\
\; w/ top-4, 8 experts & 19.7 & 13.2 & 12.6 \\ 
\; w/ top-2, 4 experts & 13.1 & 9.9 & 14.4 \\ 
\; w/ top-2, 16 experts & 32.7 & 9.9 & 12.0 \\ \bottomrule
\end{tabular}}
\label{tbl:ablation}
\end{table}


We now test various modifications to the proposed MossNet architecture. We study the {\color{black}\emph{relative}} effect of removing MHA, MLP-MoE, and varying the number of total and activated experts. Table~\ref{tbl:ablation} summarizes the results. The proposed MossNet-8x8M achieves a PPL of 13.1 with 19.7M total parameters and 9.9M active parameters, demonstrating the effective use of MHA and MLP-MoE. Removing MHA increases parameters count and leads to modest performance drop (PPL = 13.5), while removing MLP-MoE yields fewer parameters but also worse perplexity (13.4). 

Next, we test the effect of varying the number of total and activated experts. The table highlights a trade-off: activating fewer experts (e.g., top-1) can greatly hurt perplexity (up to 15.3), while activating more experts (e.g., top-4 with 8 experts) can reduce PPL to 12.6, at the cost of a higher active parameter count (resulting in higher memory and compute). Notably, employing 16 experts with top-2 activation gives the best perplexity (12.0) but increases the total parameter count to 32.7M, illustrating how scaling the MoE approach can yield lower perplexity with more overall model capacity.

\section{Conclusion}
\label{sec:conclusion}

In this paper, we introduced \textbf{MossNet}, a mixture-of-state-space-experts architecture designed to emulate a linear MHA within an SSM. By integrating MoE in both channel-mixing (MLP) and time-mixing (SSM) components, MossNet captures different temporal focus or scale of context, providing a richer representation than a single set of SSM parameters could. This is akin to the MHA mechanism in transformers. Our theoretical analysis shows that this approach indeed recovers a linearized form of MHA, and our empirical study on language modeling and downstream tasks demonstrates that MossNet outperforms both transformer-based and prior SSM/GRM-based baselines. Large-scale experiments further highlight its scalability and practical runtime benefits. We believe MossNet represents an important step toward fully harnessing recurrent models for language modeling at scale, opening up new directions for efficient, flexible, and high-performing LLM architectures.

\section{Limitations}
\label{sec:limitations}

Despite the several advantages of the proposed MossNet architecture, there are several limitations. First, the integration of the MoE framework within both channel-mixing and time-mixing components of state-space models introduces considerable architectural complexity. This may present challenges for replication and broader adoption in the research community without specialized knowledge. MoEs do not effectively improve inference performance on server, when the input is a batch of user requests containing different tasks. Further, we evaluate MossNet on MLP tasks. We leave evaluation on more diverse downstream tasks such as multi-modal understanding, real-time applications, and specialized domains to future work. Finally, although MossNet shows promising results on mobile devices like the Samsung Galaxy S24 Ultra, performance across other hardware configurations, especially those with different architectures or constraints, may vary. Future work could explore adaptive optimizations tailored to specific hardware platforms.

\bibliography{biblio}
\bibliographystyle{acl_natbib}

\appendix

\begin{table*}[t!]
\centering
\caption{Key model architecture hyperparameters and training recipes for various baseline architectures (Pythia, Llama, Mistral, Mixtral, Griffin, Mamba, Mamba2, Zamba, and MoE-Mamba) alongside the proposed MossNet family of models. The table displays model sizes, dimensions, training tokens, context lengths, and learning rate schedules, among other relevant settings. $\alpha$ corresponds to the weight factor for load balancing loss. *Unlike MossNet-8x200M+ that was dynamically trained in top-2 and top-3 modes, smaller models were only trained in top-2 mode.}
\resizebox{\linewidth}{!}{
\begin{tabular}{@{}l|cc|ccccccc|ccccccc@{}}
\toprule
\textbf{Model} & \rot[90]{\textbf{Tot. Param. (M)}} & \rot[90]{\textbf{Act. Param. (M)}} & \rot[90]{\textbf{Hidden Dim.}} & \rot[90]{\textbf{Intermediate Dim.}} & \rot[90]{\textbf{Num. Attn. Heads}} & \rot[90]{\textbf{Num. K/V Heads}} & \rot[90]{\textbf{Sliding Window}} & \rot[90]{\textbf{Num. Layers}} & \rot[90]{\textbf{Tie Embeddings}} & \rot[90]{\textbf{Train. Tokens (B)}} & \rot[90]{\textbf{Context Length}} & \rot[90]{\textbf{Max. LR}} & \rot[90]{\textbf{Schedule}} & \rot[90]{\textbf{Warmup}} & \rot[90]{\textbf{Final LR Ratio}} & \rot[90]{\textbf{$\alpha$}} \\ \midrule
Pythia-9M & 8.8 & 8.8 & 128 & 512 & 2 & - & - & 12 & T & 22 & 2048 & 2.0e-3 & Cosine & 3\% & 0\% & - \\
Llama-8M & 8.2 & 8.2 & 128 & 448 & 2 & 1 & - & 8 & T & 22 & 2048 & 2.0e-3 & Cosine & 3\% & 0\% & - \\
Mistral-8M & 8.2 & 8.2 & 128 & 448 & 2 & 1 & 256 & 8 & T & 22 & 2048 & 2.0e-3 & Cosine & 3\% & 0\% & - \\
Mixtral-8x8M & 17.8 & 9.6 & 128 & 448 & 2 & 1 & 256 & 8 & T & 22 & 2048 & 2.0e-3 & Cosine & 3\% & 0\% & 0.001 \\
Griffin-9M & 8.7 & 8.7 & 128 & 384 & 2 & 1 & - & 16 & T & 22 & 2048 & 1.0e-2 & Cosine & 3\% & 0\% & - \\
Mamba-8M & 8.3 & 8.3 & 128 & 448 & - & - & - & 16 & T & 22 & 2048 & 1.0e-2 & Cosine & 3\% & 0\% & - \\
Mamba2-9M & 8.6 & 8.6 & 128 & 256 & - & - & - & 16 & T & 22 & 2048 & 1.0e-2 & Cosine & 3\% & 0\% & - \\
Zamba-8M & 8.1 & 8.1 & 128 & 448 & 2 & 1 & - & 16 & T & 22 & 2048 & 1.0e-2 & Cosine & 3\% & 0\% & - \\
MoE-Mamba-8x8M & 16.8 & 9.7 & 128 & 384 & - & - & - & 16 & T & 22 & 2048 & 1.0e-2 & Cosine & 3\% & 0\% & 0.001 \\
\rowcolor{gray!20} 
MossNet-8x8M* & 19.7 & 9.9 & 128 & 384 & 2 & 1 & - & 16 & T & 22 & 2048 & 1.0e-2 & Cosine & 3\% & 0\% & 0.001 \\ \midrule
Pythia-22M & 22.3 & 22.3 & 256 & 1,024 & 4 & - & - & 12 & T & 22 & 2048 & 1.0e-3 & Cosine & 3\% & 0\% & - \\
Llama-20M & 20 & 20 & 256 & 896 & 4 & 2 & - & 8 & T & 22 & 2048 & 1.0e-3 & Cosine & 3\% & 0\% & - \\
Mistral-20M & 20 & 20 & 256 & 896 & 4 & 2 & 256 & 8 & T & 22 & 2048 & 1.0e-3 & Cosine & 3\% & 0\% & - \\
Mixtral-8x20M & 58.5 & 25.5 & 256 & 896 & 4 & 2 & 256 & 8 & T & 22 & 2048 & 1.0e-3 & Cosine & 3\% & 0\% & 0.001 \\
Griffin-22M & 22 & 22 & 256 & 768 & 4 & 2 & - & 16 & T & 22 & 2048 & 5.0e-3 & Cosine & 3\% & 0\% & - \\
Mamba-20M & 19.9 & 19.9 & 256 & 896 & - & - & - & 16 & T & 22 & 2048 & 5.0e-3 & Cosine & 3\% & 0\% & - \\
Mamba2-20M & 20.4 & 20.4 & 256 & 512 & - & - & - & 16 & T & 22 & 2048 & 5.0e-3 & Cosine & 3\% & 0\% & - \\
Zamba-20M & 19.5 & 19.5 & 256 & 896 & 4 & 2 & - & 16 & T & 22 & 2048 & 5.0e-3 & Cosine & 3\% & 0\% & - \\
MoE-Mamba-8x20M & 54.1 & 25.8 & 256 & 768 & - & - & - & 16 & T & 22 & 2048 & 5.0e-3 & Cosine & 3\% & 0\% & 0.001 \\
\rowcolor{gray!20} 
MossNet-8x20M* & 63.9 & 26.1 & 256 & 768 & 4 & 2 & - & 16 & T & 22 & 2048 & 5.0e-3 & Cosine & 3\% & 0\% & 0.001 \\ \midrule
Pythia-64M & 63.6 & 63.6 & 512 & 2048 & 8 & - & - & 12 & T & 22 & 2048 & 1.0e-3 & Cosine & 3\% & 0\% & - \\
Llama-67M & 66.7 & 66.7 & 512 & 1792 & 8 & 2 & - & 12 & T & 22 & 2048 & 1.0e-3 & Cosine & 3\% & 0\% & - \\
Mistral-67M & 66.7 & 66.7 & 512 & 1792 & 8 & 2 & 256 & 12 & T & 22 & 2048 & 1.0e-3 & Cosine & 3\% & 0\% & - \\
Mixtral-8x67M & 320.6 & 105.9 & 512 & 1792 & 8 & 2 & 256 & 12 & T & 22 & 2048 & 1.0e-3 & Cosine & 3\% & 0\% & 0.001 \\
Griffin-61M & 61.1 & 61.1 & 512 & 1792 & 8 & 2 & - & 16 & T & 22 & 2048 & 5.0e-3 & Cosine & 3\% & 0\% & - \\
Mamba-66M & 66.4 & 66.4 & 512 & 1792 & - & - & - & 24 & T & 22 & 2048 & 5.0e-3 & Cosine & 3\% & 0\% & - \\
Mamba2-67M & 67.2 & 67.2 & 512 & 1024 & - & - & - & 24 & T & 22 & 2048 & 5.0e-3 & Cosine & 3\% & 0\% & - \\
Zamba-62M & 62.1 & 62.1 & 512 & 1792 & 8 & 2 & - & 24 & T & 22 & 2048 & 5.0e-3 & Cosine & 3\% & 0\% & - \\
MoE-Mamba-8x66M & 272.6 & 102.8 & 512 & 1536 & 8 & 2 & - & 24 & T & 22 & 2048 & 5.0e-3 & Cosine & 3\% & 0\% & 0.001 \\
\rowcolor{gray!20} 
MossNet-8x66M* & 325.9 & 102.9 & 512 & 1536 & 8 & 2 & - & 24 & T & 22 & 2048 & 5.0e-3 & Cosine & 3\% & 0\% & 0.001 \\ \midrule
Pythia-330M & 328.6 & 328.6 & 1024 & 4096 & 16 & - & - & 22 & T & 22 & 2048 & 3.0e-4 & Cosine & 3\% & 0\% & - \\
Llama-350M & 351.4 & 351.4 & 1024 & 3584 & 16 & 4 & - & 22 & T & 22 & 2048 & 3.0e-4 & Cosine & 3\% & 0\% & - \\
Mistral-350M & 351.4 & 351.4 & 1024 & 3584 & 16 & 4 & 512 & 22 & T & 22 & 2048 & 3.0e-4 & Cosine & 3\% & 0\% & - \\
Griffin-330M & 330.4 & 330.4 & 1024 & 3584 & 16 & 4 & - & 32 & T & 22 & 2048 & 1.5e-3 & Cosine & 3\% & 0\% & - \\
Mamba-370M & 371.5 & 371.5 & 1024 & 3584 & - & - & - & 48 & T & 22 & 2048 & 1.5e-3 & Cosine & 3\% & 0\% & - \\
Mamba2-370M & 369.9 & 369.9 & 1024 & 2048 & - & - & - & 48 & T & 22 & 2048 & 1.5e-3 & Cosine & 3\% & 0\% & - \\
Zamba-330M & 334.1 & 334.1 & 1024 & 3584 & 16 & 4 & - & 48 & T & 22 & 2048 & 1.5e-3 & Cosine & 3\% & 0\% & - \\ \midrule 
\rowcolor{gray!20} 
MossNet-8x200M+ & 1554.5 & 477/657 & 1024 & 3072 & 16 & 4 & 2048 & 30 & F & 2800 & 4096 & 2.0e-4 & WSD & 1\% & 10\% & 0.001 \\ \bottomrule
\end{tabular}}
\label{tbl:arch_hyperparameters}
\end{table*}

\section{Model Hyperparameters and Training Recipes}
\label{app:hyperparameters}

In this section, we provide details on the various model architecture hyperparameters and corresponding training recipes for the MossNet suite of models and baselines at different parameter scales.

Table~\ref{tbl:arch_hyperparameters} summarizes the design choices. Each row corresponds to a particular model variant, sorted by approximate total parameter count. The key columns indicate:
\begin{itemize}
    \item Total and active parameters.
    \item Hidden and intermediate dimensions for the MLP layers.
    \item Total number of attention heads and K/V heads (for grouped-query attention;~\citealt{gqa}).
    \item Sliding window size, if applied.
    \item Number of layers.
    \item Tie embeddings, i.e., whether input and output embeddings are tied.
    \item Total number of training tokens.
    \item Context length used for training.
    \item Other training hyperparameters, including $\alpha$, i.e., the weight factor used for load balancing loss in MoE architectures.
\end{itemize}
We group models by approximate size categories, illustrating how scaling up parameters impacts the choice of dimensionality and training regimes. 

Note that we train MossNet-8x200M+ in a dynamic setting. We train the model in top-3 mode for 900 steps and in top-2 mode for 100 steps and repeat the cycle. All models are trained on the Cosmopedia dataset (fair training setting), except MossNet-8x200M+ that we train on a custom pretraining dataset.

\section{Custom Pre-training Dataset}
\label{app:custom_dataset}

Table~\ref{tbl:pretraining_dataset} shows the mixture of open datasets that form the custom pre-training data mix for MossNet-8x200M+. We combine DCLM-baseline~\citep{dclm}, Starcoder~\citep{starcoder}, Proof-Pile-2~\citep{proofpile2}, peS2o~\citep{dolma}, and Cosmopedia-2~\citep{smollmcorpus} with different sampling weights.

\begin{table}[t!]
\caption{Composition of pre-training data for MossNet-8x200M+.}
\centering
\resizebox{\linewidth}{!}{
\begin{tabular}{@{}lccc@{}}
\toprule
\textbf{Dataset} & \textbf{Num. Tok. (B)} & \textbf{Train Tok. (B)} & \textbf{Samp. Wgt.} \\ \midrule
DCLM-baseline & 4000 & 2520 & 0.90 \\
Starcoder & 250 & 168 & 0.06 \\
Proof-Pile-2 & 55 & 28 & 0.01 \\
peS2o & 47 & 28 & 0.01 \\
Cosmopedia-2 & 28 & 56 & 0.02 \\ \midrule
\multicolumn{2}{l}{Total} & 2800 & 1.00 \\ \bottomrule
\end{tabular}}
\label{tbl:pretraining_dataset}
\end{table}

\section{Additional Results}

In this section, we present additional results.

\subsection{Commonsense Performance of Fairly-trained Models}

\begin{table*}[t]
\caption{Zero-shot performance of MossNet and fairly-trained baselines on commonsense tasks.}
\centering
\resizebox{0.85\linewidth}{!}{
\begin{tabular}{@{}l|ccccccc|c@{}}
\toprule
\textbf{Model} & \textbf{ARC-c} & \textbf{ARC-e} & \textbf{COPA} & \textbf{HellaSwag} & \textbf{OBQA} & \textbf{PIQA} & \textbf{WinoGrande} & \textbf{Average} \\ \midrule
Pythia-9M       & 18.3          & 32.6          & \textbf{57.0} & 26.5          & \textbf{14.6} & 56.2          & 50.8          & 36.6          \\
Llama-8M        & 18.8          & 34.0          & 55.0          & 26.5          & 11.6          & 56.6          & \textbf{52.8} & 36.5          \\
Mistral-8M      & 19.2          & 33.5          & 56.0          & 26.4          & 11.2          & 56.4          & 50.1          & 36.1          \\
Mixtral-8x8M    & 19.5 & \textbf{35.0} & 53.0          & 27.1          & 13.0          & 56.8          & 50.3          & 36.4          \\
Griffin-9M      & 19.1          & 32.2          & 48.0          & 26.0          & 14.6          & 56.2          & 50.8          & 35.3          \\
Mamba-8M        & 19.3          & 34.0          & 51.0          & 26.4          & 13.2          & 58.0          & 51.5          & 36.2          \\
Mamba2-9M        & \textbf{19.8}          & 34.8          & 52.0          & 26.5          & 12.2          & 57.3          & 51.0          & 36.1          \\
Zamba-8M        & 18.1          & 33.2          & 53.0          & 26.2          & 14.0          & 57.4          & 50.9          & 36.1          \\
MoE-Mamba-8x8M  & 18.5          & 34.7          & 53.0          & 25.8          & 15.2          & 55.0          & 47.4          & 35.6          \\
\rowcolor{gray!20} 
MossNet-8x8M      & 18.6          & 34.5          & 55.0          & \textbf{27.7} & 14.0         & \textbf{59.6} & 49.8          & \textbf{37.1} \\ \midrule
Pythia-22M      & 19.8          & 37.8          & 63.0          & 27.1          & 13.4          & 56.6          & 50.6          & 38.3          \\
Llama-20M       & 17.9          & 35.4          & 56.0          & 26.9          & 11.2          & 58.2          & 52.1          & 36.8          \\
Mistral-20M     & 20.0          & 35.9          & 54.0          & 26.8          & 13.8          & 58.8          & 50.4          & 37.1          \\
Mixtral-8x20M   & 19.5          & 36.2          & 54.0          & 28.2          & 14.4          & 59.1          & 48.9          & 37.2          \\
Griffin-22M     & 20.2          & 29.5          & \textbf{57.0} & 26.5          & 14.4          & 56.4          & 49.7          & 36.2          \\
Mamba-20M       & 18.1          & \textbf{39.0} & 55.0          & 27.0          & 15.2          & 59.5          & 51.4          & 37.9          \\
Mamba2-20M       & 18.3          & 36.2 & 55.0          & 27.3          & 14.4          & 57.9          & 50.8          & 37.1          \\
Zamba-20M       & 20.3          & 38.5          & 55.0          & 27.2          & 14.4          & 58.6          & \textbf{52.3} & 38.0          \\
MoE-Mamba-8x20M & 19.8          & 38.1          & 53.0          & 27.5          & 16.8          & 58.8          & 48.2          & 37.5          \\
\rowcolor{gray!20} 
MossNet-8x20M     & \textbf{20.7} & 38.8          & \textbf{57.0} & \textbf{28.8} & \textbf{15.6} & \textbf{61.2} & 50.5          & \textbf{38.9} \\ \midrule
Pythia-64M      & 20.4          & 41.5          & 51.0          & 28.3         & 15.6          & 61.8          & \textbf{51.5} & 38.6          \\
Llama-67M       & 21.1          & 41.8          & 56.0          & 28.8          & 17.2          & 60.3          & 51.3          & 39.5          \\
Mistral-67M     & 20.6          & 40.7          & 57.0          & 28.2          & 17.4          & 61.8          & 49.7          & 39.4          \\
Mixtral-8x67M   & 22.4          & \textbf{46.2} & \textbf{58.0} & 31.3          & 17.0          & 63.2          & 50.5          & 41.2          \\
Griffin-61M     & 21.9          & 30.2          & 56.0          & 26.8          & 16.8          & 57.7          & 51.0          & 37.2          \\
Mamba-66M       & 21.2          & 44.4          & 48.0          & 29.0          & 17.4          & 63.0          & 50.8          & 39.1          \\
Mamba2-67M       & 21.2          & 44.2          & 56.0          & 29.2          & 16.8          & 62.3          & 45.6          & 39.3          \\
Zamba-62M       & 21.3          & 40.8          & 54.0          & 28.8          & 15.6          & 61.1          & 49.6          & 38.7          \\
MoE-Mamba-8x66M & 22.4          & 45.2          & 56.0          & 30.4          & 18.6          & 61.8          & 48.7          & 40.4          \\
\rowcolor{gray!20} 
MossNet-8x66M     & \textbf{22.6} & 45.1          & \textbf{58.0} & \textbf{31.5} & \textbf{19.4} & \textbf{65.5} & 51.0          & \textbf{41.9} \\ \midrule
Pythia-330M     & 21.8          & 45.2          & 55.0          & 29.5          & 16.6          & 62.4          & 51.1          & 40.2          \\
Llama-350M      & 23.2          & 45.2          & 58.0          & 30.6          & 16.2          & 63.2          & 50.1          & 40.9          \\
Mistral-350M    & 20.9          & 46.1          & 55.0          & 30.4          & 19.2          & 61.6          & 50.1          & 40.5          \\
Griffin-330M    & 20.3          & 36.7          & 58.0          & 29.4          & 19.4          & 61.0          & 49.1          & 39.1          \\
Mamba-370M      & 23.9          & \textbf{55.8} & 59.0 & 33.0 & 20.2 & \textbf{66.5} & 51.6          & \textbf{44.3} \\
Mamba2-370M      & \textbf{25.4}          & 50.4 & \textbf{60.0} & \textbf{33.1} & \textbf{21.2} & 65.1 & 51.9          & 43.9 \\
Zamba-330M      & 24.2 & 48.9          & 56.0          & 32.6          & 19.4          & 64.6          & \textbf{53.3} & 42.7          \\ \bottomrule
\end{tabular}}
\label{tbl:fair_commonsense_eval}
\end{table*}

Fig.~\ref{fig:commonsense_ppl_scaling} summarizes the commonsense performance of MossNet and baseline models. Table~\ref{tbl:fair_commonsense_eval} presents the detailed results. Again, MossNet outperforms baseline architectures at different active parameter scales. {\color{black} We scale parameter sizes up to 100M and leave experiments on larger models to future work.}

\subsection{Speed and Memory Results}




\begin{table*}
\centering
\caption{Memory (GB) of various models on A100-80 GPU (F16 precision, FlashAttention 2) across varying prompt lengths. Batch sizes are denoted as \textbf{1} and \textbf{4}.}
\label{tbl:gpu_eff_memory}
\resizebox{\linewidth}{!}{%
\begin{tabular}{@{}cl|ccccccc|ccccccc@{}}
\toprule
 & \textbf{Model} & \multicolumn{7}{c|}{\textbf{Batch Size 1}} & \multicolumn{7}{c}{\textbf{Batch Size 4}} \\
\cmidrule(lr){3-9}\cmidrule(l){10-16}
 & & \textbf{512} & \textbf{1024} & \textbf{2048} & \textbf{4096} & \textbf{8192} & \textbf{16384} & \textbf{32768} & \textbf{512} & \textbf{1024} & \textbf{2048} & \textbf{4096} & \textbf{8192} & \textbf{16384} & \textbf{32768} \\
\midrule
\multirow{3}{*}{\rotatebox[origin=c]{90}{$\sim$500M}} & Mamba-500M & 2.5 & 2.5 & 2.6 & 2.8 & 3.4 & 4.5 & 8.9 & 2.7 & 2.8 & 3.2 & 4.1 & 6.0 & 10.6 & 32.8 \\
 & Llama3-500M & 3.6 & 3.8 & 4.2 & 4.9 & 6.4 & 9.5 & 15.6 & 4.2 & 4.9 & 6.4 & 9.5 & 15.6 & 28.0 & 52.5 \\ 
 & \cellcolor{gray!20} MossNet-8x200M+ (top-2) 
   & \cellcolor{gray!20} 4.6   & \cellcolor{gray!20} 4.7   & \cellcolor{gray!20} 4.8   & \cellcolor{gray!20} 5.0 
   & \cellcolor{gray!20} 5.5   & \cellcolor{gray!20} 6.5   & \cellcolor{gray!20} 8.4   & \cellcolor{gray!20} 4.8
   & \cellcolor{gray!20} 5.1   & \cellcolor{gray!20} 5.5   & \cellcolor{gray!20} 6.5   & \cellcolor{gray!20} 8.4
   & \cellcolor{gray!20} 12.2  & \cellcolor{gray!20} 19.9 \\
\midrule
\multirow{3}{*}{\rotatebox[origin=c]{90}{$\sim$700M}} & Mamba-700M & 3.1 & 3.2 & 3.3 & 3.6 & 4.3 & 6.1 & 10.0 & 3.3 & 3.6 & 4.2 & 5.6 & 8.4 & 15.6 & 33.5 \\ 
 & Llama3-700M & 4.3 & 4.5 & 4.8 & 5.6 & 7.2 & 10.6 & 17.1 & 4.8 & 5.6 & 7.2 & 10.6 & 17.1 & 29.6 & 56.6 \\ 
 & \cellcolor{gray!20} MossNet-8x200M+ (top-3) 
   & \cellcolor{gray!20} 4.6   & \cellcolor{gray!20} 4.7   & \cellcolor{gray!20} 4.8   & \cellcolor{gray!20} 5.1
   & \cellcolor{gray!20} 5.6   & \cellcolor{gray!20} 6.7   & \cellcolor{gray!20} 8.8   & \cellcolor{gray!20} 4.8
   & \cellcolor{gray!20} 5.1   & \cellcolor{gray!20} 5.6   & \cellcolor{gray!20} 6.7   & \cellcolor{gray!20} 8.8
   & \cellcolor{gray!20} 13.0  & \cellcolor{gray!20} 23.9 \\ 
\midrule 
\multirow{4}{*}{\rotatebox[origin=c]{90}{$\sim$1.5B}} & Mamba-1.5B & 4.6 & 4.6 & 4.8 & 5.2 & 6.1 & 8.3 & 12.6 & 4.9 & 5.2 & 6.0 & 7.7 & 11.1 & 18.7 & 37.2 \\ 
 & Rene-1.3B & 4.6 & 4.6 & 4.6 & 4.6 & 5.1 & 6.4 & 9.2 & 4.6 & 4.8 & 5.3 & 6.4 & 9.2 & 13.8 & 23.3 \\ 
 & Hymba-1.5B  & 5.0 & 5.3 & 5.9 & 6.9 & 8.4 & 12.0 & 19.0 & 6.1 & 7.2 & 8.5 & 12.2 & 19.2 & 33.3 & 61.5 \\
 & Llama3-1.5B & 4.3 & 4.5 & 4.8 & 5.6 & 7.2 & 10.6 & 17.1 & 4.8 & 5.6 & 7.2 & 10.6 & 17.1 & 29.6 & 56.6 \\
\bottomrule
\end{tabular}%
}
\end{table*}

\begin{table*}
\centering
\caption{Prefill speed ($\times 10^3$ tok/s) of various models on A100-80 GPU (F16 precision, FlashAttention 2) across varying prompt lengths. Batch sizes are denoted as \textbf{1} and \textbf{4}.}
\label{tbl:gpu_eff_latency}
\resizebox{\linewidth}{!}{%
\begin{tabular}{@{}cl|ccccccc|ccccccc@{}}
\toprule
 & \textbf{Model} & \multicolumn{7}{c|}{\textbf{Batch Size 1}} & \multicolumn{7}{c}{\textbf{Batch Size 4}} \\
\cmidrule(lr){3-9}\cmidrule(l){10-16}
 & & \textbf{512} & \textbf{1024} & \textbf{2048} & \textbf{4096} & \textbf{8192} & \textbf{16384} & \textbf{32768} & \textbf{512} & \textbf{1024} & \textbf{2048} & \textbf{4096} & \textbf{8192} & \textbf{16384} & \textbf{32768} \\
\midrule
\multirow{3}{*}{\rotatebox[origin=c]{90}{$\sim$500M}} & Mamba-500M & 8.6 & 17.8 & 29.3 & 35.6 & 41.0 & 44.4 & 46.5 & 28.0 & 34.5 & 38.8 & 43.4 & 46.7 & 47.6 & 46.9 \\ 
 & Llama3-500M & 14.3 & 27.8 & 48.2 & 65.7 & 76.1 & 88.7 & 94.6 & 47.2 & 74.3 & 89.8 & 93.0 & 95.8 & 96.8 & 96.4 \\ 
 & \cellcolor{gray!20} MossNet-8x200M+ (top-2)
   & \cellcolor{gray!20} 10.8 & \cellcolor{gray!20} 21.6 & \cellcolor{gray!20} 40.4 & \cellcolor{gray!20} 48.6 & \cellcolor{gray!20} 52.7 & \cellcolor{gray!20} 64.0 & \cellcolor{gray!20} 75.8
   & \cellcolor{gray!20} 40.6 & \cellcolor{gray!20} 49.5 & \cellcolor{gray!20} 54.0 & \cellcolor{gray!20} 66.1 & \cellcolor{gray!20} 76.9 & \cellcolor{gray!20} 84.0 & \cellcolor{gray!20} 87.1 \\
\midrule
\multirow{3}{*}{\rotatebox[origin=c]{90}{$\sim$700M}} & Mamba-700M & 11.8 & 23.6 & 29.3 & 33.9 & 36.9 & 40.0 & 40.8 & 28.1 & 31.8 & 36.2 & 39.2 & 41.0 & 41.0 & 40.1 \\
 & Llama3-700M & 18.7 & 34.6 & 55.2 & 67.7 & 72.6 & 76.4 & 78.3 & 56.8 & 71.8 & 76.5 & 78.0 & 78.6 & 78.9 & 78.4 \\ 
 & \cellcolor{gray!20} MossNet-8x200M+ (top-3)
   & \cellcolor{gray!20} 11.3 & \cellcolor{gray!20} 22.1 & \cellcolor{gray!20} 33.4 & \cellcolor{gray!20} 38.5 & \cellcolor{gray!20} 44.8 & \cellcolor{gray!20} 57.7 & \cellcolor{gray!20} 65.4
   & \cellcolor{gray!20} 33.9 & \cellcolor{gray!20} 39.2 & \cellcolor{gray!20} 41.8 & \cellcolor{gray!20} 59.0 & \cellcolor{gray!20} 67.9 & \cellcolor{gray!20} 71.7 & \cellcolor{gray!20} 74.3 \\ 
\midrule 
\multirow{4}{*}{\rotatebox[origin=c]{90}{$\sim$1.5B}} & Mamba-1.5B & 9.9 & 15.3 & 18.5 & 20.2 & 23.2 & 24.8 & 25.2 & 17.9 & 19.7 & 22.2 & 23.4 & 24.3 & 24.2 & 23.8 \\
 & Rene-1.3B & 7.7 & 15.4 & 31.0 & 42.0 & 44.2 & 46.6 & 47.0 & 31.0 & 50.1 & 48.3 & 48.7 & 48.1 & 47.6 & 47.4 \\ 
 & Hymba-1.5B & 6.9 & 13.5 & 16.2 & 19.1 & 21.9 & 22.6 & 22.3 & 15.4 & 17.9 & 19.9 & 22.6 & 23.7 & 23.4 & 22.4 \\ 
 & Llama3-1.5B & 10.8 & 18.8 & 26.8 & 30.6 & 35.1 & 36.5 & 37.1 & 29.8 & 34.7 & 37.5 & 37.1 & 37.4 & 37.2 & 36.8 \\ 
\bottomrule
\end{tabular}%
}
\end{table*}

\begin{table*}
\centering
\caption{Generation speed (tok/s) of various models on A100-80 GPU (F16 precision, FlashAttention 2) across varying prompt lengths. Batch sizes are denoted as \textbf{1} and \textbf{4}.}
\label{tbl:gpu_eff_throughput}
\resizebox{\linewidth}{!}{%
\begin{tabular}{@{}cl|ccccccc|ccccccc@{}}
\toprule
 & \textbf{Model} & \multicolumn{7}{c|}{\textbf{Batch Size 1}} & \multicolumn{7}{c}{\textbf{Batch Size 4}} \\
\cmidrule(lr){3-9}\cmidrule(l){10-16}
 & & \textbf{512} & \textbf{1024} & \textbf{2048} & \textbf{4096} & \textbf{8192} & \textbf{16384} & \textbf{32768} & \textbf{512} & \textbf{1024} & \textbf{2048} & \textbf{4096} & \textbf{8192} & \textbf{16384} & \textbf{32768} \\
\midrule
\multirow{3}{*}{\rotatebox[origin=c]{90}{$\sim$500M}} & Mamba-500M & 17.9 & 18.3 & 17.7 & 18.0 & 18.1 & 17.8 & 17.9 & 68.3 & 70.5 & 71.8 & 69.8 & 70.1 & 70.4 & 70.6 \\ 
 & Llama3-500M & 28.8 & 29.2 & 27.7 & 28.2 & 27.6 & 26.2 & 22.7 & 110.4 & 113.5 & 108.3 & 101.8 & 92.2 & 76.0 & 55.9 \\ 
 & \cellcolor{gray!20} MossNet-8x200M+ (top-2)
   & \cellcolor{gray!20} 27.5 & \cellcolor{gray!20} 28.6 & \cellcolor{gray!20} 27.0 & \cellcolor{gray!20} 28.1 & \cellcolor{gray!20} 27.9 & \cellcolor{gray!20} 27.9 & \cellcolor{gray!20} 27.7
   & \cellcolor{gray!20} 91.2 & \cellcolor{gray!20} 92.5 & \cellcolor{gray!20} 91.1 & \cellcolor{gray!20} 90.3 & \cellcolor{gray!20} 89.5 & \cellcolor{gray!20} 86.5 & \cellcolor{gray!20} 81.0 \\ 
\midrule
\multirow{3}{*}{\rotatebox[origin=c]{90}{$\sim$700M}} & Mamba-700M & 24.5 & 24.7 & 25.0 & 24.8 & 24.7 & 24.6 & 25.2 & 98.0 & 95.1 & 95.7 & 97.4 & 96.6 & 94.56 & 99.0 \\ 
 & Llama3-700M & 37.1 & 40.1 & 39.0 & 38.9 & 37.4 & 33.8 & 29.6 & 158.0 & 151.0 & 144.0 & 133.5 & 113.5 & 87.6 & 61.0 \\ 
 & \cellcolor{gray!20} MossNet-8x200M+ (top-3)
   & \cellcolor{gray!20} 28.7 & \cellcolor{gray!20} 28.6 & \cellcolor{gray!20} 28.1 & \cellcolor{gray!20} 28.3 & \cellcolor{gray!20} 28.0 & \cellcolor{gray!20} 27.4 & \cellcolor{gray!20} 26.9
   & \cellcolor{gray!20} 90.2 & \cellcolor{gray!20} 91.7 & \cellcolor{gray!20} 92.4 & \cellcolor{gray!20} 91.6 & \cellcolor{gray!20} 87.1 & \cellcolor{gray!20} 85.2 & \cellcolor{gray!20} 53.0 \\ 
\midrule 
\multirow{4}{*}{\rotatebox[origin=c]{90}{$\sim$1.5B}} & Mamba-1.5B & 20.7 & 21.1 & 21.0 & 21.1 & 20.6 & 20.8 & 20.6 & 81.3 & 82.0 & 82.4 & 80.8 & 81.9 & 82.0 & 81.7 \\ 
 & Rene-1.3B & 27.3 & 27.0 & 27.3 & 27.0 & 27.0 & 26.6 & 26.7 & 107.6 & 105.5 & 106.5 & 106.9 & 107.6 & 107.4 & 108.8 \\ 
 & Hymba-1.5B & 14.9 & 14.9 & 14.6 & 14.9 & 14.5 & 14.6 & 14.6 & 58.4 & 56.3 & 57.0 & 57.2 & 56.3 & 56.6 & 56.4 \\ 
 & Llama3-1.5B & 22.2 & 22.4 & 21.2 & 21.4 & 20.9 & 18.3 & 15.0 & 88.3 & 85.7 & 80.5 & 69.7 & 58.5 & 43.3 & 28.2 \\ 
\bottomrule
\end{tabular}%
}
\end{table*}

Figs.~\ref{fig:gpu_speed_memory} and \ref{fig:device_speed_memory} summarize the speed and memory performance of MossNet-8x200M+ and baseline models at different active parameters scales. We present the detailed results for GPU profiling in Tables~\ref{tbl:gpu_eff_memory}, \ref{tbl:gpu_eff_latency}, and \ref{tbl:gpu_eff_throughput} for memory consumption, prefill speed, and generation speed, respectively. We also present the detailed results for mobile profiling in Tables \ref{tbl:device_eff_memory}, \ref{tbl:device_eff_latency}, and \ref{tbl:device_eff_throughput}.

\begin{table*}
\centering
\caption{Memory (MB) of various models on S24 Ultra (Q8 precision) across varying prompt lengths and active parameters. Batch size set to 1. MossNet on-device results reported without SWA implemented.}
\label{tbl:device_eff_memory}
\resizebox{0.65\linewidth}{!}{%
\begin{tabular}{@{}cl|ccccccc@{}}
\toprule
 & \textbf{Model} & \textbf{512} & \textbf{1024} & \textbf{2048} & \textbf{4096} & \textbf{8192} & \textbf{16384} & \textbf{32768} \\
\midrule
\multirow{3}{*}{\rotatebox[origin=c]{90}{$\sim$500M}} & Mamba-500M & 673.7 & 673.7 & 673.7 & 673.7 & 673.7 & 673.7 & 673.7 \\
 & Llama3-500M & 610.6 & 621.6 & 649.6 & 747.4 & 995.4 & 1491.4 & 2483.4 \\ 
 & \cellcolor{gray!20} MossNet-8x200M+ (top-2)
   & \cellcolor{gray!20} 1666.9 & \cellcolor{gray!20} 1667.7 & \cellcolor{gray!20} 1670.7 & \cellcolor{gray!20} 1720.7 & \cellcolor{gray!20} 1872.7 & \cellcolor{gray!20} 2176.7 & \cellcolor{gray!20} 2784.7 \\ 
\midrule
\multirow{3}{*}{\rotatebox[origin=c]{90}{$\sim$700M}} & Mamba-700M & 919.3 & 919.3 & 919.3 & 919.3 & 919.3 & 919.3 & 919.3 \\ 
 & Llama3-700M & 901.6 & 917.6 & 942.6 & 1046.4 & 1302.4 & 1814.4 & 2838.4 \\ 
 & \cellcolor{gray!20} MossNet-8x200M+ (top-3)
   & \cellcolor{gray!20} 1707.0 & \cellcolor{gray!20} 1709.9 & \cellcolor{gray!20} 1711.9 & \cellcolor{gray!20} 1761.8 & \cellcolor{gray!20} 1913.8 & \cellcolor{gray!20} 2217.8 & \cellcolor{gray!20} 2825.8 \\
\midrule
\multirow{2}{*}{\rotatebox[origin=c]{90}{$\sim$1.5B}} & Mamba-1.5B & 1748.0 & 1748.0 & 1748.0 & 1748.0 & 1748.0 & 1748.0 & 1748.0 \\ 
 & Llama3-1.5B & 1657.2 & 1694.2 & 1760.2 & 1950.0 & 2374.0 & 3222.0 & OOM \\ 
\bottomrule
\end{tabular}%
}
\end{table*}

\begin{table*}
\centering
\caption{Memory (MB) of various models on S24 Ultra (Q8 precision) across varying prompt lengths and active parameters. Batch size set to 1. MossNet on-device results reported without SWA implemented.}
\label{tbl:device_eff_latency}
\resizebox{0.65\linewidth}{!}{%
\begin{tabular}{@{}cl|ccccccc@{}}
\toprule
 & \textbf{Model} & \textbf{512} & \textbf{1024} & \textbf{2048} & \textbf{4096} & \textbf{8192} & \textbf{16384} & \textbf{32768} \\
\midrule
\multirow{3}{*}{\rotatebox[origin=c]{90}{$\sim$500M}} & Mamba-500M & 673.7 & 673.7 & 673.7 & 673.7 & 673.7 & 673.7 & 673.7 \\
 & Llama3-500M & 610.6 & 621.6 & 649.6 & 747.4 & 995.4 & 1491.4 & 2483.4 \\ 
 & \cellcolor{gray!20} MossNet-8x200M+ (top-2)
   & \cellcolor{gray!20} 1666.9 & \cellcolor{gray!20} 1667.7 & \cellcolor{gray!20} 1670.7 & \cellcolor{gray!20} 1720.7 & \cellcolor{gray!20} 1872.7 & \cellcolor{gray!20} 2176.7 & \cellcolor{gray!20} 2784.7 \\ 
\midrule
\multirow{3}{*}{\rotatebox[origin=c]{90}{$\sim$700M}} & Mamba-700M & 919.3 & 919.3 & 919.3 & 919.3 & 919.3 & 919.3 & 919.3 \\ 
 & Llama3-700M & 901.6 & 917.6 & 942.6 & 1046.4 & 1302.4 & 1814.4 & 2838.4 \\ 
 & \cellcolor{gray!20} MossNet-8x200M+ (top-3)
   & \cellcolor{gray!20} 1707.0 & \cellcolor{gray!20} 1709.9 & \cellcolor{gray!20} 1711.9 & \cellcolor{gray!20} 1761.8 & \cellcolor{gray!20} 1913.8 & \cellcolor{gray!20} 2217.8 & \cellcolor{gray!20} 2825.8 \\
\midrule
\multirow{2}{*}{\rotatebox[origin=c]{90}{$\sim$1.5B}} & Mamba-1.5B & 1748.0 & 1748.0 & 1748.0 & 1748.0 & 1748.0 & 1748.0 & 1748.0 \\ 
 & Llama3-1.5B & 1657.2 & 1694.2 & 1760.2 & 1950.0 & 2374.0 & 3222.0 & OOM \\ 
\bottomrule
\end{tabular}%
}
\end{table*}

\begin{table*}
\centering
\caption{Prefill speed (tok/s) of various models on S24 Ultra (Q8 precision) across varying prompt lengths and active parameters. Batch size set to 1. MossNet on-device results reported without SWA implemented.}
\label{tbl:device_eff_throughput}
\resizebox{0.58\linewidth}{!}{%
\begin{tabular}{@{}cl|ccccccc@{}}
\toprule
 & \textbf{Model} & \textbf{512} & \textbf{1024} & \textbf{2048} & \textbf{4096} & \textbf{8192} & \textbf{16384} & \textbf{32768} \\
\midrule
\multirow{3}{*}{\rotatebox[origin=c]{90}{$\sim$500M}} & Mamba-500M & 79 & 72 & 70 & 67 & 65 & 62 & 59 \\
 & Llama3-500M & 107 & 87 & 65 & 45 & 24 & 13 & 6 \\ 
 & \cellcolor{gray!20} MossNet-8x200M+ (top-2)
   & \cellcolor{gray!20} 120 & \cellcolor{gray!20} 107 & \cellcolor{gray!20} 101 & \cellcolor{gray!20} 92 & \cellcolor{gray!20} 74 & \cellcolor{gray!20} 57 & \cellcolor{gray!20} 36 \\ 
\midrule
\multirow{3}{*}{\rotatebox[origin=c]{90}{$\sim$700M}} & Mamba-700M & 56 & 52 & 51 & 50 & 46 & 45 & 46 \\ 
 & Llama3-700M & 71 & 58 & 48 & 35 & 21 & 12 & 6 \\ 
 & \cellcolor{gray!20} MossNet-8x200M+ (top-3)
   & \cellcolor{gray!20} 74 & \cellcolor{gray!20} 73 & \cellcolor{gray!20} 67 & \cellcolor{gray!20} 62 & \cellcolor{gray!20} 59 & \cellcolor{gray!20} 42 & \cellcolor{gray!20} 31 \\
\midrule
\multirow{2}{*}{\rotatebox[origin=c]{90}{$\sim$1.5B}} & Mamba-1.5B & 26 & 24 & 23 & 22 & 22 & 21 & 21 \\ 
 & Llama3-1.5B & 31 & 26 & 21 & 15 & 9 & 5 & OOM \\ 
\bottomrule
\end{tabular}%
}
\end{table*}

\subsection{Long-context Performance}

\begin{table*}[]
\caption{Long context performance of MossNet and various baselines. Training tokens and whether SWA is implemented for corresponding models are also provided. *Perplexity should not be compared directly for different models as they could be using different tokenizers; trend with context size should be observed instead.}
\centering
\resizebox{0.5\linewidth}{!}{
\begin{tabular}{@{}l|cccccc@{}}
\toprule
\multirow{2}{*}{\textbf{Model}} & \multirow{2}{*}{\textbf{Train Tok.}} & \multirow{2}{*}{\textbf{SWA}} & \multicolumn{4}{c}{\textbf{PPL*}} \\ \cmidrule(l){4-7} 
 &  &  & \textbf{2048} & \textbf{4096} & \textbf{8192} & \textbf{16384} \\ \midrule
Phi1.5 & 0.15T & F & 11.8 & 70.6 & 441.1 & OOM \\
SmolLM2-1.7B & 11T & F & 9.8 & 83.6 & 590.8 & OOM \\
Mamba-1.4B & 0.3T & - & 6.9 & 6.7 & 7.2 & OOM \\
\rowcolor{gray!20} MossNet-8x200M+ & 2.8T & T & 8.6 & 8.2 & 8.0 & 8.7 \\ \bottomrule
\end{tabular}}
\label{tbl:long_context}
\end{table*}

Table~\ref{tbl:long_context} presents the long context performance of MossNet-8x200M+ and various baselines. We observe that architectures using SSM and/or SWA backbones do not lose perplexity as the context size is increased. This confines with the observations of \citet{samba}.

\subsection{Choosing $k$}

{\color{black} Table~3 varies the number of routed experts~($k$) while keeping all other hyper‑parameters fixed.  The main observations are:

\begin{itemize}
  \item \textbf{Large first step, then saturation.}  
        With eight experts, increasing $k$ from~1 to~2 reduces perplexity from~15.3 to~13.1 ($-2.2$), whereas a further increase to $k=4$ only improves perplexity to~12.6 ($-0.5$) while significantly increasing active parameter count ($+3.3$).
  \item \textbf{Pool size matters.}  
        Holding $k=2$ and shrinking the pool from 8 to 4 experts worsens perplexity (13.1 to 14.4).  Conversely, expanding the pool to 16 experts (still routing $k=2$) attains the best perplexity (12.0) \emph{without} increasing \emph{active} parameters, although
        the \emph{total} model size grows, resulting in a larger disk size.
\end{itemize}

We propose the following practical rule-of-thumb:
\[
  k \;=\;
  \min \!\Bigl( 2,\; \bigl\lfloor N_{\text{experts}} / 4 \bigr\rfloor \Bigr)
\]
\vspace{-0.5em}

This caps compute at~$\le 2\times$ the dense baseline, preserving MossNet’s on‑device speed advantage. It also maintains high router entropy; choose $k=1$ only when $N_{\text{experts}}<8$. Finally, it secures $\ge 80\%$ of the achievable perplexity gain while avoiding the potential latency hit for $k\ge4$.}

\subsection{Computational complexity of MoE blocks (and MossNet)}

{\color{black} MossNet replaces the dense channel‐/time‐mixing layers in Mamba with \textit{standard} top‑$k$ MoE blocks. Because only $k$ experts are executed per token, its \emph{time} complexity is
$\Theta(L\,k\,d\,d_{\text{ff}})$ and its \emph{activation memory} is
$\Theta(k\,d)$—identical to other MoE architectures for the same $k$. Thus MossNet offers the usual “capacity without extra compute” benefit of MoE while preserving the linear‑time, constant‑cache profile of its dense counterpart.}

\end{document}